\documentclass[11pt]{article} \pdfoutput=1
\usepackage{amsmath,amssymb,amsthm,amstext,mathtools}
\usepackage{xspace}
\usepackage{url}
\usepackage[hidelinks]{hyperref}
\usepackage{authblk}

\advance \oddsidemargin by -0.8in \newcommand{\myparskip}{3pt}
\parskip \myparskip

\usepackage{subcaption}

\newtheorem{theorem}{Theorem}
\newtheorem{corollary}{Corollary}
\newtheorem{definition}{Definition}
\newtheorem{lemma}{Lemma}

\usepackage{graphicx} 

\usepackage{layouts}
\usepackage[ruled,vlined,linesnumbered]{algorithm2e}

\newcommand{\model}[3]{\ensuremath{\mathcal{G}{\left({#1},{#2},{#3}\right)}}}
\def\ea{(1+1)~EA\xspace} \DeclareMathOperator{\E}{E}

\usepackage{todonotes}

\title{Fixed-Parameter Tractability of the (1+1)~Evolutionary
  Algorithm on Random Planted Vertex Covers}

\author[1]{Jack Kearney} \author[2]{Frank Neumann} \author[1]{Andrew
  M. Sutton}

\affil[1]{%
  Algorithmic Evolution Lab\protect\\
  Department of Computer Science\protect\\
  University of Minnesota Duluth\protect\\~}

\affil[2]{%
  Optimisation and Logistics Group\protect\\
  School of Computer Science\protect\\
  University of Adelaide}

\begin{document}
\maketitle

\begin{abstract}
  We present the first parameterized analysis of a standard (1+1)
  Evolutionary Algorithm on a distribution of vertex cover
  problems. We show that if the planted cover is at most logarithmic,
  restarting the (1+1)~EA every $O(n \log n)$ steps will find a cover
  at least as small as the planted cover in polynomial time for
  sufficiently dense random graphs $p > 0.71$. For superlogarithmic
  planted covers, we prove that the (1+1)~EA finds a solution in
  fixed-parameter tractable time in expectation.

  We complement these theoretical investigations with a number of
  computational experiments that highlight the interplay between
  planted cover size, graph density and runtime.
\end{abstract}

\section{Introduction}
\label{sec:introduction}
Combinatorial problems with planted solutions have been an important
subject of study on a wide range of settings. In this scenario, a
fixed solution is hidden within a large random structure such as a
graph. The canonical example of this is the \emph{planted clique}
problem where a fixed complete subgraph of size $k$ is placed within a
large Erd\H{o}s-R\'{e}nyi random graph on $n \gg k$ vertices. The task
is to either recover the hidden
solution~\cite{AlonKrivelevichSudakov1998HiddenClique} or one of size
at least $k$~\cite{Jerrum1992cliques}.  These problems have important
applications in
cryptography~\cite{JuelsPeinado2000HidingCliquesCryptographic} for
example. In the context of randomized search heuristics,
Storch~\cite{Storch2007cliques} investigated the planted clique
problem for random local search (RLS) and the (1+1)~EA\@.  More
recently, Doerr et al.~\cite{DoerrEtAl2017TimeComplexityAnalysis}
considered randomly generated propositional satisfiability problems
with planted assignments and proved that the (1+1)~EA requires at most
$O(n \log n)$ time to solve this problem provided that the constraint
density is high enough.

Planted vertex covers have recently been studied in the context of
\emph{systematically incomplete
  data}~\cite{BensonKleinberg2018FoundGraphData} in networks. In this
view, true node interactions can only be observed among some core set
$C$, whereas a potentially much larger set of fringe nodes lies
outside this sphere of observability. This may occur, for example, in
social networks and communication data
sets~\cite{RomeroUzziKleinberg2019SocialNetworksUnderStress} where a
company only knows about links within the company and between an
employee and the outside world, but not about links between external
entities.  This translates to a planted vertex cover problem on a
graph $G = (V,E)$. An adversary knows of a subset $C \subseteq V$
which is a vertex cover, and the task is to identify a set as close to
$C$ as possible.

In the $\model{n}{k}{p}$ model, a graph $G = (V,E)$ is constructed on
a set $V$ of $n$ vertices by taking a size-$k$ subset $C \subseteq V$
to be the \emph{core}. An edge appears in $G$ with probability $p$
unless it connects two vertices in $V\setminus C$, in which it occurs
with probability zero. Therefore, $G$ is guaranteed to have a
$k$-vertex cover. Note that a graph can be constructed from this model
by drawing a standard Erd\H{o}s-R\'{e}nyi graph and subsequently
deleting all edges that connect fringe vertices.

This model is a special case of the \emph{stochastic block model} of
random graphs from network
theory~\cite{HollandBlackmondLeinhardt1983StochasticBlockModels} in
which the vertex set is partitioned into $r$ disjoint communities and
edge probabilities are specified by a symmetric $r \times r$ matrix
$P$ where a vertex in community $i$ is connected to a vertex in
community $j$ with probability $P_{ij}$.  The stochastic block model
allows for the generation of graphs from which the community subgraphs
might be recovered partially or in full from the graph
data~\cite{AbbeSandon2015communitydetection}. This models the
detection of \emph{community structure} in networks, which is a
fundamental problem in computer science.  The $\model{n}{k}{p}$ model
we study in this work is a stochastic block model with $r=2$ and
probability matrix
\[
  P = \left[
    \begin{array}{cc}
      p & p\\
      p & 0\\
    \end{array}
  \right].
\]

In this paper, we are interested in the performance of simple
randomized search heuristics on planted vertex cover problems in the
context of parameterized complexity. We prove that, for sufficiently
``dense'' graphs (i.e., large enough $p$), the (1+1)~EA is with high
probability a fixed-parameter tractable heuristic for the $k$-vertex
cover problem where $k$ is the size of the planted solution. More
precisely, if $k$ is at most logarithmic, we prove there is a
threshold on $p$ such that above this threshold the (1+1)~EA is very
likely to find a $k$-cover in almost linear time. For larger values of
$k$, we show that the (1+1)~EA runs in $O(f(k,p) n \log n))$ time
where $f$ is a function of $k$ and $p$ (but not $n$).

The first parameterized result on vertex cover is due to Kratsch and
Neumann~\cite{KratschNeumann2012FixedParameterEvolutionary} who
demonstrated that Global SEMO using instance-specific mutation
operators has expected optimization time
$O(OPT\cdot n^4 + n\cdot 2^{OPT^2+OPT})$ on any graph $G$ where $OPT$
is the size of the optimal vertex cover of $G$. This result can be
tightened to $O(n^2\log n + OPT\cdot n^2 + 4^{OPT}n)$ by incorporating
the cost of an optimal fractional vertex cover provided by an LP
solver into the fitness function.  A recent study by Baguley et
al.~\cite{baguley2023fixed} extended these multi-objective approaches
to the W-separator problem.  Using a special focused jump-and-repair
mechanism, Branson and
Sutton~\cite{BransonSutton2021FocusedJumpAndRepair} showed that
evolutionary algorithms can solve the vertex cover problem in expected
time $O(2^{OPT} n^2\log n)$ by probabilistically simulating an
iterative compression routine.

The above results hold for all graphs $G$ with vertex cover size
$OPT$. In this paper, we sacrifice the generality of the problem
slightly in order to investigate a more general algorithm, i.e., the
(1+1)~EA\@. To our knowledge, we present here the first parameterized
complexity result on vertex cover problems for a standard evolutionary
algorithm that does not rely on any special mutation operators.

\medskip
\noindent \textbf{Our results.}
For random planted graph models with $n$
vertices, edge density $p$ and planted cover size $k$, we show that if
$k \le \ln n$, then if $p > \sqrt{\frac{1-\ln \delta}{2}}$ for any
constant $\delta \in (1/e,1)$, a restart framework for the \ea finds a
$k$-cover in $n^{c+1} \log n$, where $c$ is a constant. If
$k > \ln n$, then we show for any $0 < p < 1$, the expected time of
the (1+1)~EA is
$O{\left(k^{4k\left(1+\frac{1}{p}\right)} n \log n\right)}$, i.e., the
\ea runs in FPT time parameterized by $k$ and $p$.

We also provide the results of computational experiments that
investigate regimes that our theorem does not cover, for example when
both $p$ and $k$ are small. These results elucidate the relationship
between $k$ and $p$ and the runtime of the \ea, and hint at new
interesting directions for future theoretical study.

\section{Preliminaries}
\label{sec:preliminaries}
Given a graph $G = (V,E)$ on $n$ vertices, we encode subsets of $V$ as
elements of $\{0,1\}^n$ in the usual way. For $x \in \{0,1\}^n$,
denote as $|x|$ as the number of bits set to $1$ in $x$ (i.e., the
cardinality of the set to which it corresponds). The fitness function
typically employed by evolutionary algorithms on the minimum vertex
cover problem first penalizes infeasible sets (sets that do not cover
all edges in $E$), then penalizes larger feasible covers:
\begin{equation}
  \label{eq:fitness}
  f(x) = |x| + n\cdot\Big\lvert \Big\{ (u,v) \in E \colon x[u] = x[v] = 0 \Big\}\Big\rvert.
\end{equation}
This fitness function is quite natural for searching for a minimal
cover, and was originally designed by Khuri and
B\"{a}ck~\cite{Khuri94anevolutionary}. It has been studied extensively
both empirically and
theoretically~\cite{Khuri94anevolutionary,OlivetoEtAl2009AnalysisVC,FriedrichEtAl2010ApproximatingCoveringProblems}.

We point out that this is a so-called \emph{vertex-based}
representation for which there are currently no bounds on the
approximation ratio for the (1+1)~EA\@. It is possible to obtain a
guaranteed 2-approximation with the (1+1)~EA by using edge-based
representations instead~\cite{DBLP:conf/foga/JansenOZ13}. This is
rather notable, as minimum vertex cover is likely hard to approximate
below a $(2-\epsilon)$ factor~\cite{KhotRegev2008Vertexcovermight}.

\begin{algorithm}  
  \KwIn{A fitness function $f \colon \{0,1\}^n \to \mathbb{R}$}

  Choose $x$ uniformly at random from $\{0,1\}^n$\; \While{termination
    criteria not met}{Create $y$ by flipping each bit of $x$ with
    probability $1/n$\; \lIf{$f(y) \le f(x)$}{$x \gets y$} }
  \Return{$x$}\;
  \caption{\label{alg:ea} (1+1)~EA}
\end{algorithm}

Many of our theoretical results make use of multiplicative drift with
tail bounds, which we state in the following theorem for reference.
\begin{theorem}[Multiplicative Drift
  \cite{DoerrGoldberg2010DriftAnalysisTail,KoetzingKrejca2019firsthitting}]
  \label{thm:multi-drift}
  Let $(X_t)_{t \in \mathbb{N}}$ be a stochastic process over
  $\mathbb{R}$, $x_{\min} > 0$ and let
  $T \coloneqq \min\{t : X_t < x_{\min}\}$. Suppose that
  $X_0 \ge x_{\min}$ and, for all $t \le T$, it holds that
  $X_t \ge 0$, and there exists some $\delta > 0$ such that, for all
  $t < T$, $\E[X_t - X_{t+1} \mid X_0,\ldots,X_t] \ge \delta X_t$,
  then,
  \begin{enumerate}
  \item $\E[T \mid X_0] \le \frac{\ln(X_0/x_{\min})+1}{\delta}$, and
  \item
    $\Pr\left(T \ge \frac{\ln(X_0/x_{\min})+r}{\delta}\right) \le
    e^{-r}$
  \end{enumerate}
\end{theorem}

The fitness function in Equation~\eqref{eq:fitness} ensures that
Algorithm~\ref{alg:ea} quickly finds a feasible cover, which is
captured in Theorem~\ref{thm:feasible}, which was proved
asymptotically in \cite[Theorem
1]{FriedrichEtAl2010ApproximatingCoveringProblems}. We restate this
result here with a simple upper bound with leading constants using
drift.

\begin{theorem}
  \label{thm:feasible}
  The expected time until the \ea finds a feasible cover for any graph
  on $n$ vertices is at most $\frac{1}{2}(e n \ln n + e n)$.
\end{theorem}
\begin{proof}
  Let $(X_t)_{t \in \mathbb{N}}$ be the stochastic process that counts
  the number of edges uncovered by the candidate solution in iteration
  $t$ of the \ea. For any vertex $u$, denote as $d_t(u)$ the count of
  uncovered edges incident to $u$ in iteration $t$. Since any vertex
  $u$ is flipped with probability $(1-1/n)^{n-1}(1/n) \ge (en)^{-1}$,
  and an increase in uncovered edges is never accepted, we may bound
  the drift of $(X_t)$ as
  \begin{align*}
    \E[X_t - X_{t+1} \mid X_t] \ge \sum_u \frac{d_t(u)}{e n} = \frac{2 X_t}{en} 
  \end{align*}
  since each of the $X_t$ uncovered edges is counted twice in the sum
  over $d_t$. The claim follows by Theorem~\ref{thm:multi-drift}.
\end{proof}

\begin{definition}
  \label{def:model}
  Let $n,k \in \mathbb{N}$ and $p \in (0,1)$. The $\model{n}{k}{p}$
  model of random planted graphs is a distribution of random graphs on
  $n$ vertices defined by construction as follows.

  Let $V$ be a set of $n$ (labeled) vertices. Choose a $k$-subset
  $C \subset V$ uniformly at random, and for each $u,v \in V$, if
  $\{u,v\} \cap C \ne \emptyset$, add edge $uv$ to $E$ with
  probability $p$.

  In the resulting graph $G = (V,E)$, we refer to $C$ as the
  \emph{core}, and each $v \in C$ as a core vertex. We refer to
  vertices in $V \setminus C$ as \emph{fringe} vertices.
\end{definition}

\section{Small $k$}
\label{sec:small-k}
In this section we consider $\model{n}{k}{p}$ where $k \le \ln n$. Our
results rely heavily on the following property of planted vertex cover
graphs, which we call \emph{$\delta$-heaviness}.

\begin{definition}
  Let $G = (V,E)$ be a graph drawn from the $\model{n}{k}{p}$
  model. For a constant $0 < \delta < 1$, we say $G$ is $\delta$-heavy
  if for every subset $S \subset V \setminus C$ where
  $|S| = \delta |V \setminus C|$, every core vertex in $C$ is adjacent
  to at least $\ln n$ vertices in $S$.
\end{definition}

\begin{lemma}
  \label{lem:delta-heavy}
  Let $G = (V,E)$ be a graph drawn from the $\model{n}{k}{p}$
  model. Let $\delta,p \in (0,1)$ be constants. If
  $p > \sqrt{\frac{1-\ln \delta}{2}}$, then $G$ is $\delta$-heavy with
  probability $1-e^{-\Omega(n)}$.
\end{lemma}
\begin{proof}
  Fix an arbitrary $v \in C$ and an arbitrary $\delta (n-k)$-sized
  subset $S \subset V \setminus C$. We first bound the probability
  that $v$ is adjacent to no more than $\ln n$ vertices in $S$.  Let
  $X$ be the random variable that counts the edges between $v$ and
  vertices in $S$.  Each edge from $v$ to a vertex in $S$ appears
  independently with probability $p$, so $X$ is the sum of $|S|$
  independent Bernoulli random variables, each with success
  probability $p$ so $\E[X] = p|S|$.  By Hoeffding's
  inequality~\cite{Hoeffding1963}, for any $t > 0$,
  $\Pr(X \le \E[X] - t) < e^{-2t^2/|S|}$, thus the probability that
  $v$ is adjacent to at most $\ln n$ vertices in $S$ can be estimated
  by
  \begin{align*}
    \Pr(X \le \ln n) &= \Pr(X \le \E[X] - (\E[X] - \ln n))\\
                     &< e^{-2(p|S|-\ln n)^2/|S|}\\
                     & =\exp\left(-2\left(p^2|S| + \frac{\ln^2 n}{|S|} - 2 p \ln n \right)\right)\\
                     &\le
                       \exp\left(-2\delta p^2 (n-k) + 4 p \ln n\right).
  \end{align*}
  We have assumed $k \le \ln n$, so this probability is at most
  \[
    \exp\left(-2\delta p^2 (n-\ln n) + 4 p \ln n\right) <
    \exp\left(-2\delta p^2 n + 6 p \ln n\right).
  \]
  Note that we have used here the fact that $\delta < 1$ and
  $p^2 < p$.  Taking a union bound over all $k$ vertices $v \in C$,
  the probability that any core vertex is adjacent to fewer than
  $\ln n$ vertices in $S$ is at most
  \[
    \exp\left(-2\delta p^2 n + 6 p \ln n + \ln k\right).
  \]
  A final union bound over all subsets $S$ of size
  $\delta |V\setminus C| = \delta(n-k)$ shows the probability that $G$
  is not $\delta$-heavy is at most
  \begin{align*}
    \binom{n}{\delta n}&\exp\left(-2\delta p^2 n + 6 p \ln n + \ln k\right)\\
                       & \le
                         \frac{e^{\delta n}n^{\delta n}}{(\delta n)^{\delta n}}\exp\left(-2\delta p^2 n + 6 p \ln n + \ln k\right)\\
                       &= \exp\left(-2\delta p^2 n + 6 p \ln n + \ln k + \delta n \ln(e/\delta)\right)\\
                       &\le \exp\left(-\delta n(2p^2 - \ln(e/\delta)) + (6p+1)\ln n\right).
  \end{align*}
  Since $p > \sqrt{\frac{1-\ln \delta}{2}}$, and $p$ and $\delta$ are
  taken to be positive constants, we have
  $2p^2 - \ln(e/\delta) = \Omega(1)$, and the probability that $G$ is
  not $\delta$-heavy is $e^{-\Omega(n)}$, which completes the proof.
\end{proof}

\begin{theorem}
  \label{thm:small-k}
  Consider the $\model{n}{k}{p}$ model with $k \le \ln n$ and
  $p > \sqrt{\frac{1-\ln \delta}{2}}$ for some constant
  $\delta \in (1/e,1)$. Then for all but an exponentially-fast
  vanishing fraction of all graphs $G$ sampled from $\model{n}{k}{p}$,
  if $T$ is the runtime for the \ea to find a $k$-cover on $G$, we
  have
  \[
    \Pr\left(T \le 2en\ln n + \lfloor e n (1-\delta) \rfloor \right) =
    \Omega(n^{-(e(1-\delta)\ln(2e) + \ln 2)}).
  \]
\end{theorem}
\begin{proof}
  Since $p$ is sufficiently large, by Lemma~\ref{lem:delta-heavy}, all
  but an $e^{-\Omega(n)}$-fraction of graphs drawn from
  $\model{n}{k}{p}$ are $\delta$-heavy. Thus, we assume for the
  remainder of the proof that $G$ is $\delta$-heavy.

  Let $\mathcal{E}$ be the event that after exactly
  $\lfloor e n (1-\delta) \rfloor$ iterations of the \ea, the
  following conditions hold:
  \begin{enumerate}
  \item The core vertices $C$ belong to the current solution of the
    \ea,
  \item There are at least $\delta n$ fringe vertices that are not
    part of the current solution of the \ea.
  \end{enumerate}
  This is a rather fortunate event for the \ea, because such a
  candidate solution is already a feasible vertex cover (as all
  vertices in $C$ are present), so after this point no infeasible
  covers would be accepted. Moreover, since $G$ is $\delta$-heavy,
  every core vertex is adjacent to at least $\ln n$ uncovered edges
  (by condition (2) above). Thus in order to remove a core vertex $v$
  from the cover, a single mutation operation would need to change at
  least $\ln n$ neighbors of $v$ to remain feasible. In contrast, it
  is always possible to remove any fringe vertex from the current
  cover. Thus if there are $i$ fringe vertices in the current
  solution, the probability to improve the fitness is at least
  $i/(en)$. Furthermore, the probability of flipping at least $\ln n$
  vertices in a single mutation is $n^{-\omega(1)}$.

  Let $\{X_t\}_{t \in \mathbb{N}}$ denote the stochastic process that
  tracks the number of fringe vertices in the cover at time $t$. The
  drift of $\{X_t\}$ conditioned on $\mathcal{E}$ and starting at
  iteration $\lfloor e n (1-\delta) \rfloor$ is at least
  $X_t/en - n^{-\omega(1)} = \Omega(X_t /n)$. By
  Theorem~\ref{thm:multi-drift},
  \[
    \Pr\left(T < 2e n \ln n + \lfloor e n (1-\delta)\rfloor \mid
      \mathcal{E}\right) = 1 - o(1)
  \]
  It remains to bound the probability of $\mathcal{E}$. Let
  $\mathcal{E}_1$ be the event that the initial solution to the \ea
  contains every vertex in $C$ and let $\mathcal{E}_2$ be the event
  that the core vertices in $C$ are not mutated during the first
  $\lfloor e n(1-\delta) \rfloor$ iterations of the \ea. Conditioning
  on $\mathcal{E}_1 \cap \mathcal{E}_2$, the \ea already starts with a
  feasible solution and does not remove any core vertices during the
  first $\lfloor e n(1-\delta) \rfloor$ steps.

  Let $T_1$ be the random variable that measures the number of
  iterations until the first time the number of fringe vertices in the
  cover drops below a $\delta$-fraction.  Again applying tail bounds
  on multiplicative drift, and noting that
  $1+\ln\left(\frac{1}{1-\delta}\right) \ge 1 - \delta$ for constant
  $0 < \delta < 1$, under the condition
  $\mathcal{E}_1 \cap \mathcal{E}_2$, the \ea has reduced the number
  of fringe vertices in the cover from at most $n-k$ to at most
  $\delta(n-k)$ with probability at least $1-1/e$. Applying the law of
  total probability we have
  \begin{align*}
    \Pr(\mathcal{E}) &\ge \Pr(\mathcal{E} \mid \mathcal{E}_1 \cap \mathcal{E}_2) \Pr(\mathcal{E}_1 \cap \mathcal{E}_2) \\
                     &= \Pr(\mathcal{E} \mid \mathcal{E}_1 \cap \mathcal{E}_2) \Pr(\mathcal{E}_2 \mid \mathcal{E}_1) \Pr(\mathcal{E}_1)\\
                     & \ge \left(1 - \frac{1}{e}\right)\cdot \left[\left(1 - \frac{1}{n}\right)^{k}\right]^{\lfloor e n(1-\delta) \rfloor} (1/2)^k\\
                     &\ge \left(1 - 1/e\right)\cdot (2e)^{-ek(1-\delta)} \cdot 2^{-k}\\
                     & \ge \left(1 - 1/e\right) \cdot n^{-(e(1-\delta)\ln(2e) + \ln 2)},
  \end{align*}
  where we have used $k \le \ln n$ in the final inequality.
\end{proof}

\begin{algorithm}
  \KwIn{A fitness function $f \colon \{0,1\}^n \to \mathbb{R}$ and a
    run length $\ell$} $t \gets 0$\; \While{termination criteria not
    met}{\If{$t = 0$}{Choose $x$ uniformly at random from
      $\{0,1\}^n$\; } Create $y$ by flipping each bit of $x$ with
    probability $1/n$\; \lIf{$f(y) \le f(x)$}{$x \gets y$}
    $t \gets (t+1) \bmod \ell$\; } \Return{$x$}\;
  \caption{\label{alg:ea-restart} \ea with cold restarts}
\end{algorithm}

Theorem~\ref{thm:small-k} provides a lower bound on the probability
that a run of length at least
$2en\ln n + \lfloor e n (1-\delta) \rfloor$ finds a $k$-cover of a
random graph with sufficient density. This bound vanishes with $n$,
but slowly enough that a simple cold-restart strategy (periodically
starting over from a randomly generated cover) is guaranteed to be
efficient. This is captured by the following corollary.

\begin{corollary}[to Theorem~\ref{thm:small-k}]
  Consider the $\model{n}{k}{p}$ model with $k \le \ln n$ and
  $0.71 \le p \le 1$. Running the \ea with cold restarts
  (Algorithm~\ref{alg:ea-restart}) with $\ell = 3e n \ln n$ finds a
  $k$-cover on all but an exponentially-fast vanishing fraction of
  graphs in $O(n^{c+1} \log n)$ fitness evaluations where
  $0.73 < c \le e(1+\ln 2)-1 < 3.61 $ is a constant depending on $p$.
\end{corollary}

\begin{proof}
  Let $\delta = e^{1 - 2p^2}$. Since $p > 0.71$, we have
  $\delta \in (1/e,1)$. Thus the conditions for
  Theorem~\ref{thm:small-k} are satisfied, and the success probability
  for an independent run of length $3e n \ln n$ of the \ea is
  $\Omega(n^{-(e(1-\delta)\ln(2e) + \ln 2)}$. Under this condition,
  the number of independent runs until a success is geometrically
  distributed with expectation
  $n^{e(1-\delta)\ln(2e) + \ln 2} = n^{e(1 - e^{1 - 2 p^2}) (1 + \ln
    2) + \ln 2}$, and $c$ can be chosen appropriately.
\end{proof}

\section{Large $k$}
\label{sec:large-k}
We now consider $\model{n}{k}{p}$ in which $k > \ln n$. We will make
use of the following probabilistic bound on the size of independent
sets in the core.

\begin{lemma}
  \label{lem:largest-independent-set}
  Suppose $G$ is drawn from the $\model{n}{k}{p}$ model with
  $k=\omega(1)$. Then with probability $1-o(1)$, the largest
  independent set in $C$ has size at most $(1+2/p)\ln k + 1$.
\end{lemma}
\begin{proof}
  Set $\ell \coloneqq \lceil (1+2/p)\ln k + 1\rceil$. There are
  $\binom{k}{\ell}$ size-$\ell$ vertex sets in $C$. We label these
  sets from $1$ to $\binom{k}{\ell}$ and consider a sequence
  $X_1,\ldots,X_{\binom{k}{\ell}}$ of indicator random variables over
  $\model{n}{k}{p}$ where
  \[
    X_i = \begin{cases} 1 & \text{if the $i$-th size-$\ell$ subset of
        $C$ is an
                            independent set in $G$,}\\
            0 & \text{otherwise.}
          \end{cases}
        \]
        Consider the sum $X = X_1 + \cdots + X_{\binom{k}{\ell}}$ and
        note that $X = 0$ if and only if there are no independent sets
        of size $\ell$ or larger in $G$. By Markov's inequality,
        \begin{align*}
          \Pr(X \ge 1) &\le \E[X] = \binom{k}{\ell} (1-p)^{\binom{\ell}{2}} \le  k^{\ell}\left( (1-p)^{(\ell-1)/2} \right)^{\ell}\\
                       & \le \left(\exp\left(\ln k - p(\ell -
                         1)/2\right)\right)^{\ell},\,\text{since $1-p \le e^{-p}$,}\\
                       &= \exp \left(-\left[\left(1+\frac{p}{2}\right)\ln k + \frac{p}{2}\right]\ln k \right)\\
                       & \le e^{-\ln^2 k},
        \end{align*}
        since $p \ge 0$.
      \end{proof}

      \begin{theorem}
        \label{thm:large-k}
        Consider a graph $G$ drawn from the $\model{n}{k}{p}$ model
        with $k > \ln n$. Then with probability $1-o(1)$ (taken over
        the model), the expected runtime of the \ea to find a cover of
        size at most $k$ on $G$ is
        $O{\left(k^{4k\left(1+\frac{1}{p}\right)} n \log n\right)}$.
      \end{theorem}
      \begin{proof}
        By Theorem~\ref{thm:feasible}, the \ea takes at most
        $\frac{1}{2}(e n \ln n + e n)$ steps in expectation to find a
        feasible solution, after which the \ea never accepts an
        infeasible solution.

        Consider the potential function $\phi(x) = \max\{0,f(x)-k\}$
        and note that when $\phi(x) = 0$, $x$ is a feasible cover of
        size at most $k$. Moreover, $\phi$ cannot increase during the
        run of the \ea.

        By Lemma~\ref{lem:largest-independent-set}, the largest
        independent set in the core of $G$ contains at most
        $(1+\frac{2}{p})\ln k + 1$ vertices with probability $1-o(1)$,
        and we condition on this event for the remainder of the
        proof. Consider the stochastic process
        $(X_t)_{t \in \mathbb{N}}$, which corresponds to the potential
        in the $t$-th iteration.

        We seek to bound the drift of $(X_t)$ after finding a feasible
        solution. Assume that the \ea has already found a feasible
        solution, and let $C$ be the core vertices of $G$. Let $x$ be
        the current solution. We make the following case distinction
        on $x$.
  
        \begin{description}
        \item[Case 1:] $C \cap \{i : x[i] = 0\} = \emptyset$. In this
          case, all of the vertices in $C$ are in the cover described
          by $x$. Thus, any fringe vertex can be removed from the
          current cover and the resulting set is still a cover. A
          particular vertex is removed from the cover with probability
          $(1/n)(1-1/n)^{n-1}$ and there are $f(x) - k$ fringe
          vertices, so the drift in this case is
          \[
            \E[X_t - X_{t+1} \mid X_t] \ge \frac{f(x)-k}{n}\left(1 -
              \frac{1}{n}\right)^{n-1} \ge \frac{X_t}{e n}.
          \]
        \item[Case 2:] $C \cap \{i : x[i] = 0\} \neq \emptyset$. In
          this case, some of the core vertices are not in the cover
          described by $x$. Let $Z \coloneqq C \cap \{x[i] = 0\}$ be
          the set of core vertices that are not in the current
          cover. Note that since $x$ is feasible $Z$ must be an
          independent set in $C$ (otherwise there would be an
          uncovered edge in $C$).

          Let $Z'$ be an arbitrary set of exactly $|Z|$ fringe
          vertices that belong to the current solution $x$, i.e.,
          $Z' \subseteq \{i : x[i] = 1 \} \cap (V \setminus C)$ with
          $|Z'| = |Z|$. Such a $Z'$ must exist, otherwise we would
          have $f(x) < k$.  Let $\mathcal{E}$ denote the event that
          mutation changes all of the zero-bits corresponding to $Z$
          into one-bits, and all of the of one-bits corresponding to
          $Z'$ to zero.  Since each bit is mutated independently, we
          may invoke the principle of deferred
          decisions~\cite{MitzenmacherUpfal2005probability} and assume
          that the choices are first made for the bits in $Z$ and $Z'$
          to produce a partially mutated offspring. Hence, we assume
          that $\mathcal{E}$ has occurred, and consider the random
          choices on the remaining bits corresponding to
          $V \setminus (Z \cap Z')$.  There are
          $f(x) - (k-|Z|) = f(x) - k + |Z|$ fringe vertices in $x$,
          and after removing $|Z'| = |Z|$ fringe vertices, there are
          still $f(x) - k = X_t$ fringe vertices that have not yet
          been considered for mutation, so we may assume that we are
          in Case 1, now with exactly $f(x)-k = X_t$ fringe vertices
          remaining in the cover. Since $X_t - X_{t+1} \ge 0$, by the
          law of total expectation, we can bound the drift from below
          as follows.
          \begin{align*}
            \E[X_t - X_{t+1} \mid X_t] &\ge \E[X_t - X_{t+1} \mid X_t \cap \mathcal{E}]\Pr(\mathcal{E})\\
                                       &\ge n^{-2|Z|} \frac{X_t}{e n},
          \end{align*}
          since $\Pr(\mathcal{E}) = n^{-(|Z|+|Z'|)} = n^{-2|Z|}$.
        \end{description}
        In either case, the drift is at least
        $n^{-2|Z|} \frac{X_t}{e n}$, but we have assumed via
        Lemma~\ref{lem:largest-independent-set} that
        $|Z| \le (1+\frac{2}{p})\ln k + 1 < 2(1+1/p)\ln k$ for
        sufficiently large $n$ (and hence $k$, as $k \ge \ln
        n$). Therefore, by the multiplicative drift theorem, the
        expected time until a $k$-cover is found is at most

\begin{align*}
  O(n^{4(1+1/p)\ln k} n \log n) &= O(k^{4(1+1/p)\ln n} n \log n)\\
                                &= O\left(k^{4k\left(1+\frac{1}{p}\right)} n \log n\right),
\end{align*}  
since $\ln n < k$.
\end{proof}

\section{Computational Experiments}
\label{sec:experiments}

To fill in the gaps left open by the previous sections, we report here
on a number of experiments that investigate the relationship between
the parameters of the planted vertex cover problem. For each
experiment, we sample from the $\model{n}{k}{p}$ model by constructing
a random graph on $n$ vertices choosing each edge with probability $p$
as long as at least one incident vertex is in the set
$\{1,\ldots,k\}$. After this, we run the standard \ea
(Algorithm~\ref{alg:ea}) until $f(x) \le k$. For each setting of $n$,
$k$, $p$, we run the algorithm for 100 trials (but sample a new graph
from $\model{n}{k}{p}$ each time.

To better understand how the runtime depends on $n$ on dense graphs in
which $k$ is a small function of $n$, we plot the average runtime,
varying $n = 100, \ldots,1000$ and fixing $p=0.5$. This is plotted in
Figure~\ref{fig:variedn-dense}, where we observe a stable runtime
varying almost linearly with $n$. In Figure~\ref{fig:variedn-sparse},
we show the same data for runs where $p$ is also varied with $n$,
i.e., $p=1/n$. This corresponds to much sparser graphs, and we see
that the runtime has much higher variability, especially for slower
growing $k$.

\begin{figure}
  \centering
  \begin{subfigure}[t]{0.495\textwidth}
    \includegraphics[width=\linewidth]{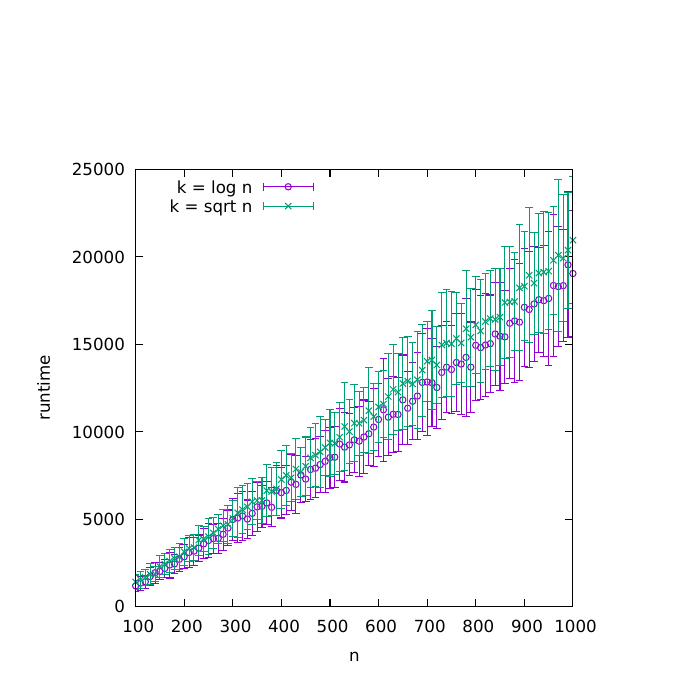}
    \caption{Dense regime ($p=0.5$).}
    \label{fig:variedn-dense}
  \end{subfigure}
  \begin{subfigure}[t]{0.495\textwidth}
    \centering
    \includegraphics[width=\linewidth]{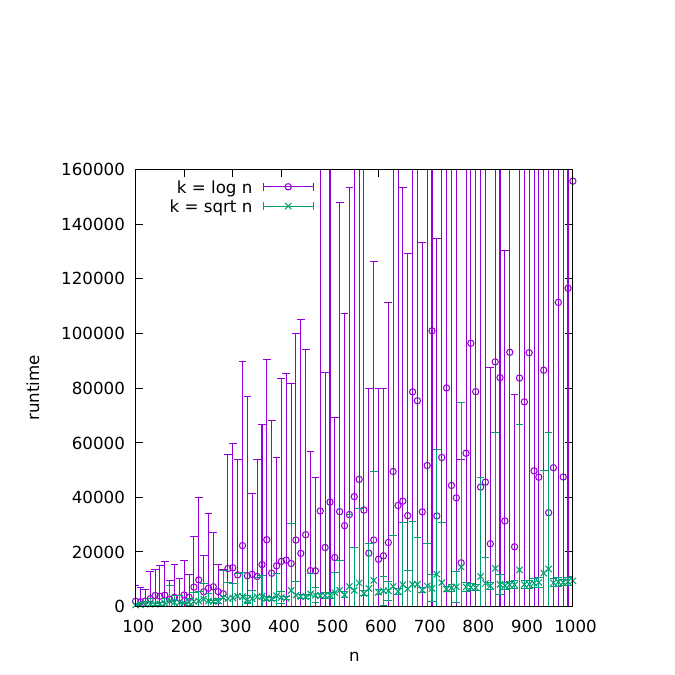}
    \caption{Sparse regime ($p=1/n$).}
    \label{fig:variedn-sparse}
  \end{subfigure}
  \caption{Runtime dependence on $n$ for $k=\ln n$ and $k=
    \sqrt{n}$. Error bars denote standard deviation.}

\end{figure}

This scaling behavior is not so surprising, as we expect that random
planted graphs are particularly easy for the \ea. Similar to the case
of random planted
satisfiability~\cite{DoerrEtAl2017TimeComplexityAnalysis}, the
relatively uniform structure of the problem is likely to provide a
good fitness signal for hill-climbing type algorithms.

Random distributions of problems often undergo a so-called phase
transition as various system parameters are varied. Very often,
problems sampled near a critical density tend to be (empirically)
harder to solve by different algorithms. For example, empirical
evidence suggests critically-constrained planted propositional
satisfiability formulas are difficult for the \ea when they are
sampled near a critical
density~\cite{DoerrEtAl2017TimeComplexityAnalysis}. To study the
performance of the \ea on $\model{n}{k}{p}$ as a function of graph
density, we plot the dependence of the average runtime on $p$ in
Figures~\ref{fig:n1000-runtime-p} and~\ref{fig:n200-runtime-p},
holding $n$ fixed and averaging over all values of $k$. We also see in
this case a dependence on graph density in which the \ea performs
worse in a band of not-too-sparse but not-too-dense graphs.

\begin{figure}[t!]
  \centering
  \begin{subfigure}[t]{0.495\textwidth}
    \includegraphics[width=\linewidth]{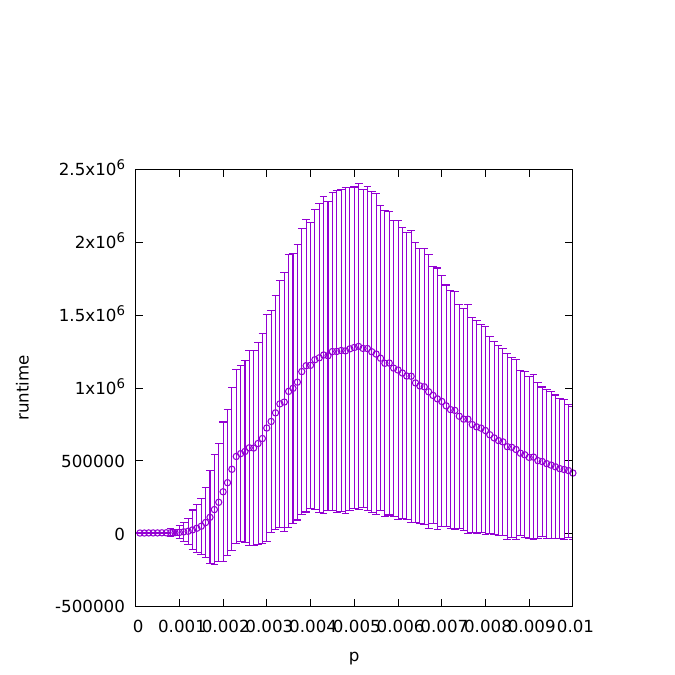}
    \caption{$n=1000$}
    \label{fig:n1000-runtime-p}
  \end{subfigure}
  \begin{subfigure}[t]{0.495\textwidth}
    \centering
    \includegraphics[width=\linewidth]{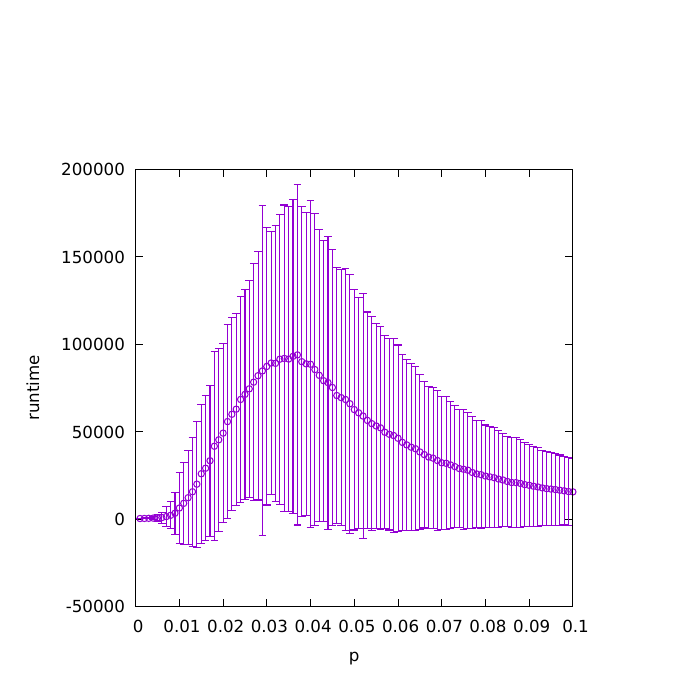}
    \caption{$n=200$}
    \label{fig:n200-runtime-p}
  \end{subfigure}
  \caption{Runtime dependence on $p$ for fixed $n$ varying
    $k=10,\ldots,100$. Error bars denote standard deviation.}
\end{figure}

The dependence of runtime on $k$, however, is more uniform as we can
see in Figure~\ref{fig:runtime-k}. Here we have aggregated over all
$p$ values, which likely explains the large variance, especially in
the larger $n=1000$ problems.

\begin{figure}
  \centering \includegraphics[width=0.495\textwidth]{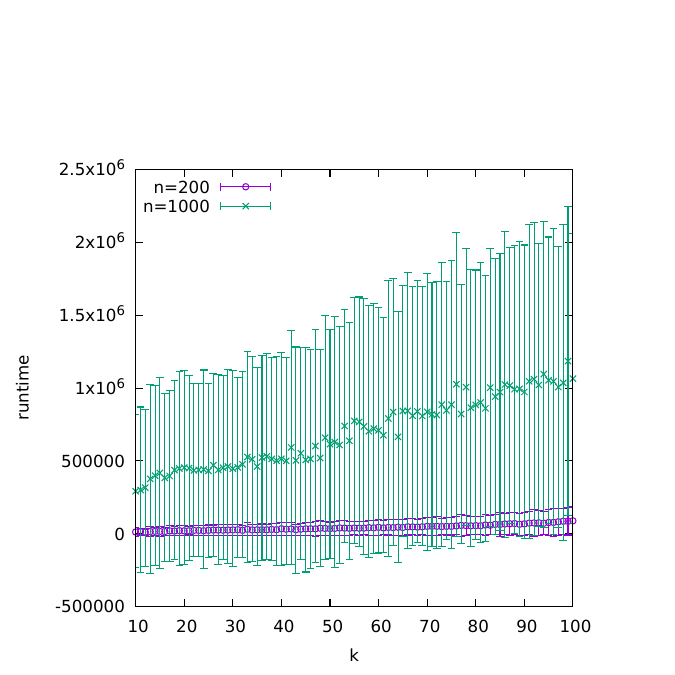}
  \caption{Runtime dependence on $k$ ($p$ aggregated). Error bars
    denote standard deviation.}
  \label{fig:runtime-k}
\end{figure}

A more detailed picture is provided by
Figures~\ref{fig:n200-runtime-kp} and~\ref{fig:n1000-runtime-kp},
where we display two-dimensional color plots showing the runtime
dependence on both $k$ and $p$ simultaneously. On these plots one can
see how the density and the cover size influences the efficiency of
the \ea. We conjecture that there is a critical value (or range) of
$p$ at which the \ea struggles to find a $k$-cover.

\begin{figure}
  \centering
  \begin{subfigure}[t]{0.495\textwidth}
    \includegraphics[width=\linewidth]{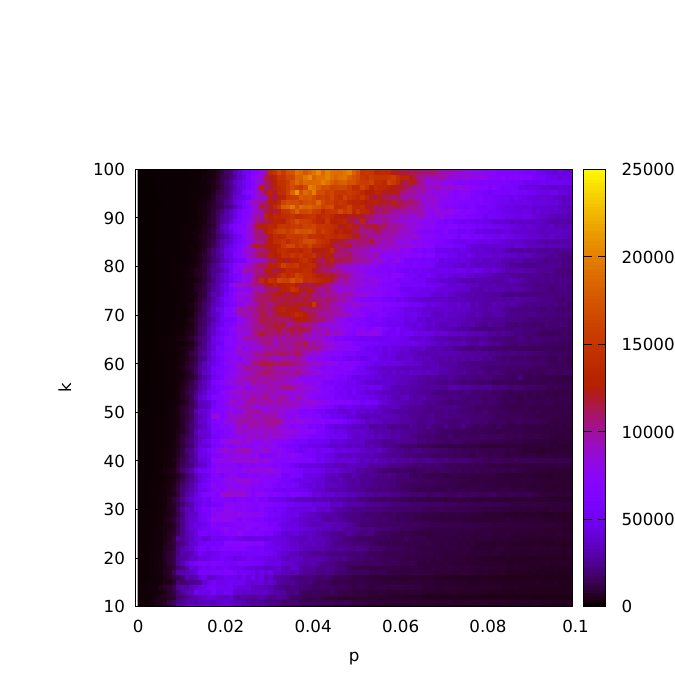}
    \caption{$n=200$}
    \label{fig:n200-runtime-kp}
  \end{subfigure}
  \begin{subfigure}[t]{0.495\textwidth}
    \centering
    \includegraphics[width=\linewidth]{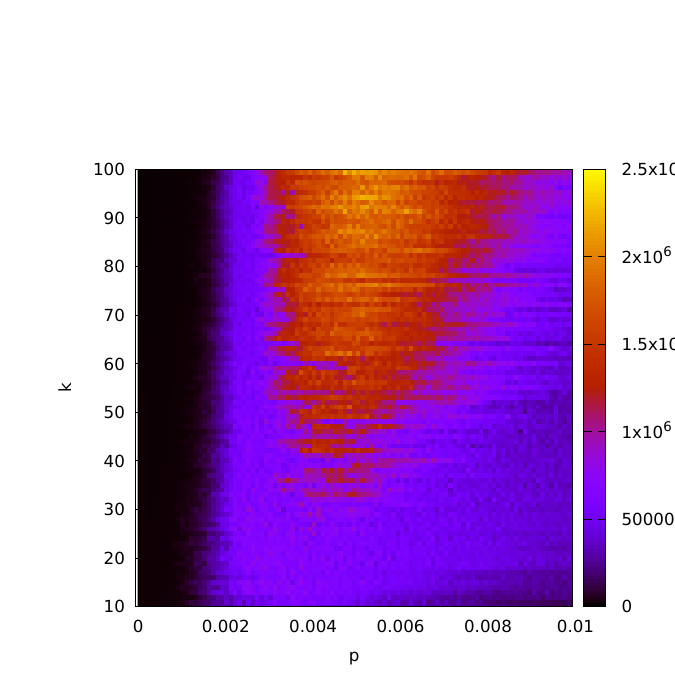}
    \caption{$n=1000$}
    \label{fig:n1000-runtime-kp}
  \end{subfigure}
  \caption{Runtime dependence on both $k$ and $p$ for fixed $n$.}
\end{figure}

The \ea completes execution as soon as it finds a $k$-cover. However,
this is not necessarily guaranteed to be the $k$-cover that was
planted in the graph. Indeed, for smaller densities, we would expect
many other $k$-covers in the graph. To investigate this, in
Figure~\ref{fig:core-p} we plot the proportion of runs in which the
planted $k$-core was recovered (as opposed to some different
$k$-cover) as a function of $p$. The dependence of this characteristic
as a function of $k$ is plotted in Figure~\ref{fig:core-k}, and
Figures~\ref{fig:n200-core-kp} and~\ref{fig:n1000-core-kp} display
this in a color plot for both $k$ and $p$ simultaneously.

\begin{figure}
  \begin{subfigure}[t]{0.495\textwidth}
    \centering \includegraphics[width=\linewidth]{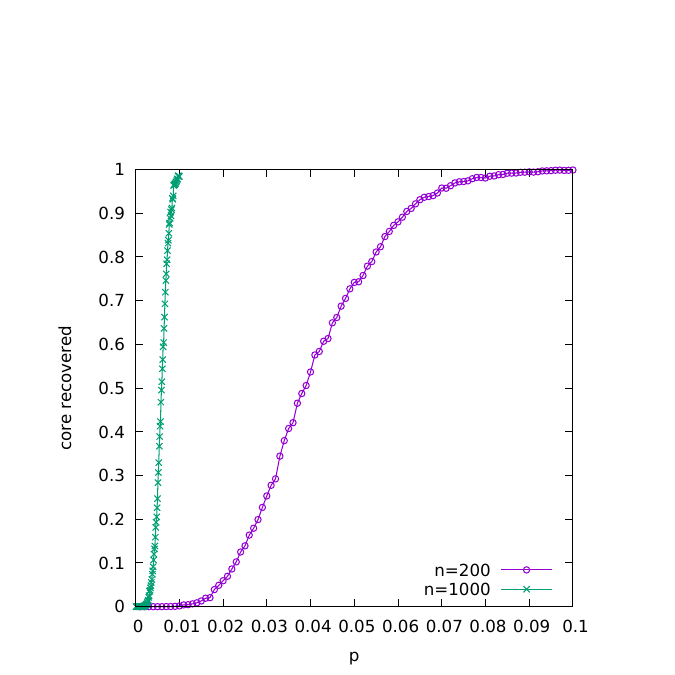}
    \caption{$k$-core recovered as a a function of $p$.}
    \label{fig:core-p}
  \end{subfigure}
  \begin{subfigure}[t]{0.495\textwidth}
    \centering \includegraphics[width=\linewidth]{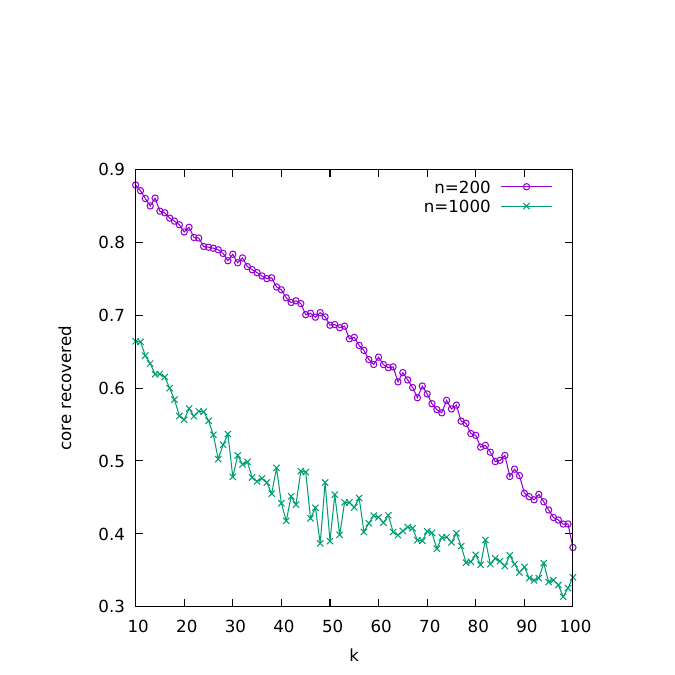}
    \caption{$k$-core recovered as a function of $k$.}
    \label{fig:core-k}
  \end{subfigure}

\begin{subfigure}[t]{0.495\textwidth}
  \centering \includegraphics[width=\linewidth]{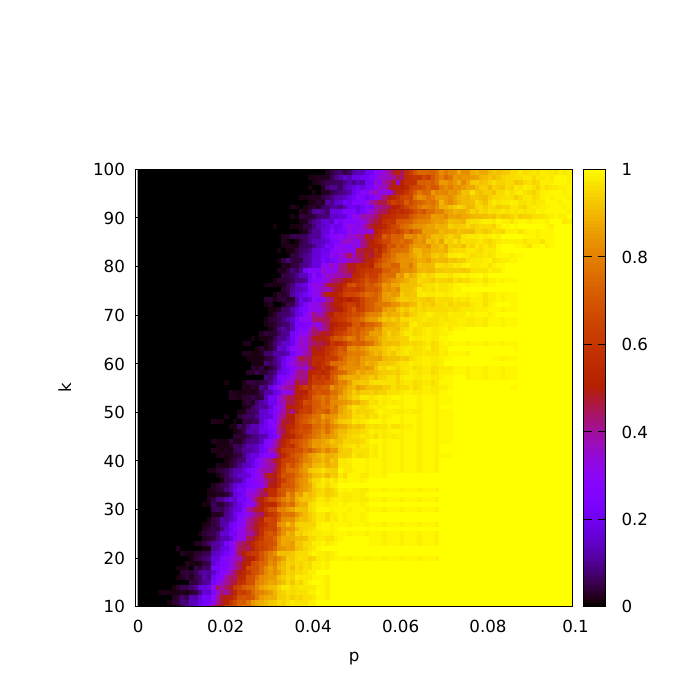}
  \caption{$k$-core recovered as a function of $k$ and $p$ ($n=200$).}
  \label{fig:n200-core-kp}
\end{subfigure}
\begin{subfigure}[t]{0.495\textwidth}
  \centering
  \includegraphics[width=\linewidth]{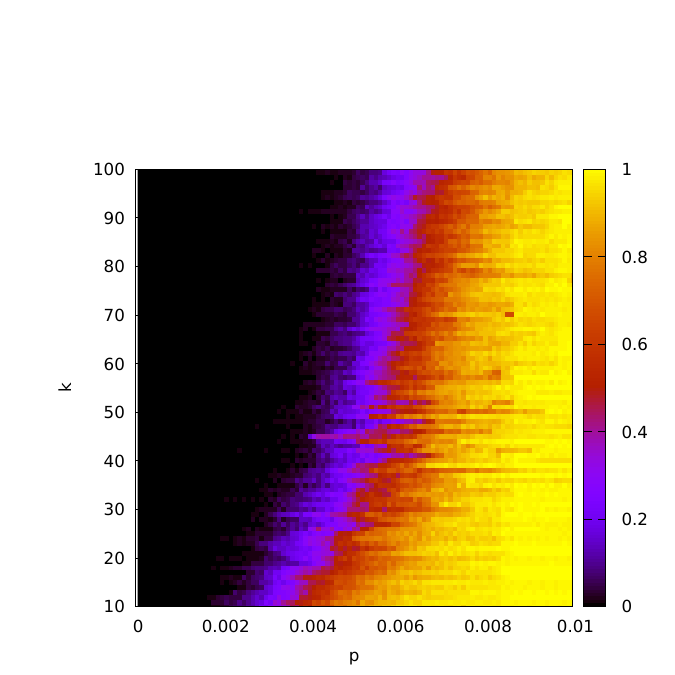}
  \caption{$k$-core recovered as a function of $k$ and $p$
    ($n=1000$).}
  \label{fig:n1000-core-kp}
\end{subfigure}
\caption{Proportion of runs in which the planted $k$-core was
  recovered.}
\end{figure}

When the graph is relatively sparse, we would also expect the \ea to
``overshoot'' $k$ by finding an even smaller cover before finding a
$k$-cover. To understand better how this depends on $k$ and $p$, we
plot the average difference between $k$ and the best fitness found as
a function of $p$ on sparse ($p=1/n$) instances where $n$ is varied in
Figure~\ref{fig:variedn-sparse-kdelta}, and on fixed-$n$ instances in
Figures~\ref{fig:n1000-kdelta} and~\ref{fig:n200-kdelta}.

\begin{figure}
  \centering
  \begin{subfigure}[t]{0.495\textwidth}
    \centering
    \includegraphics[width=\linewidth]{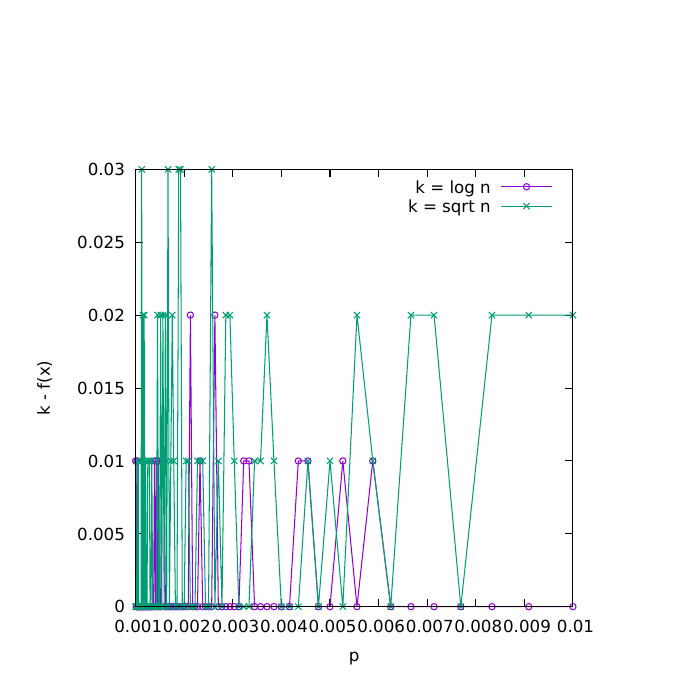}
    \caption{$n = 100, \ldots,1000$}
    \label{fig:variedn-sparse-kdelta}
  \end{subfigure}
  \begin{subfigure}[t]{0.495\textwidth}
    \centering
    \includegraphics[width=\linewidth]{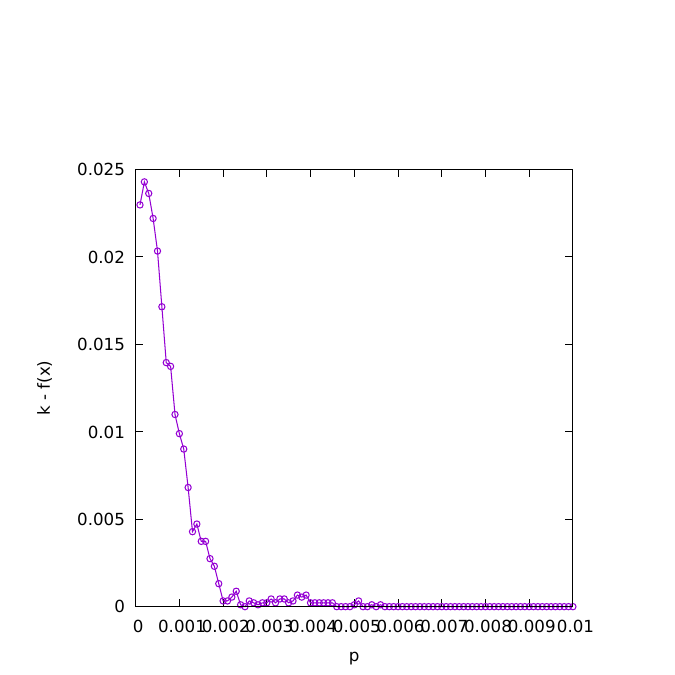}
    \caption{$n=1000$}
    \label{fig:n1000-kdelta}
  \end{subfigure}
  \begin{subfigure}[t]{0.495\textwidth}
    \centering
    \includegraphics[width=\linewidth]{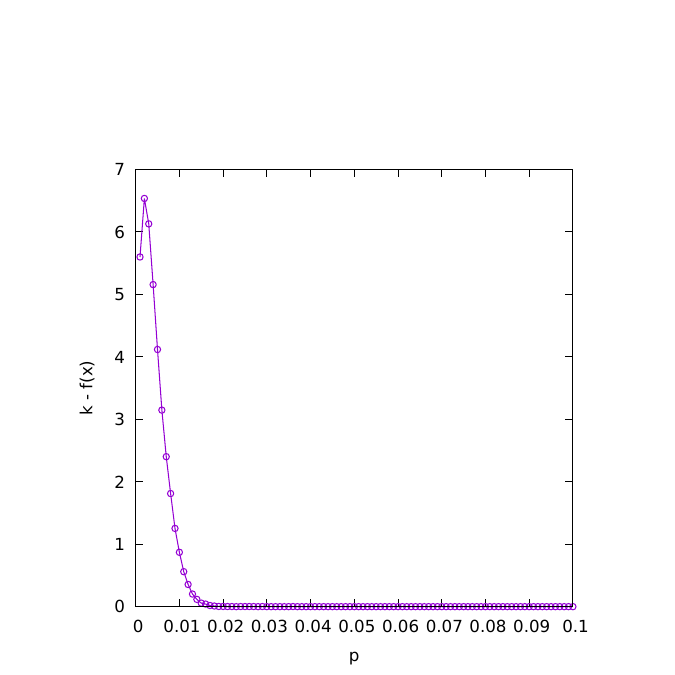}
    \caption{$n=200$}
    \label{fig:n200-kdelta}
  \end{subfigure}
  \caption{Average difference between $k$ and best fitness found as a
    function of $p$.}
\end{figure}

\section{Conclusion}
In this paper we have presented a parameterized analysis the \ea on
problems drawn from the $\model{n}{p}{k}$ random planted vertex cover
model. We showed that for dense graphs $(p > 0.71)$ and small $k$,
there is sufficient signal in enough of the space so that the \ea has
a relatively good chance of finding a $k$-cover in a polynomial-length
run. When $k$ is large, we showed that a feasible cover cannot leave
too much of the planted core uncovered, and therefore the \ea does not
require a large effort to make progress. In the end, this translates
to a fixed-parameter tractable runtime for the \ea with high
probability over $\model{n}{p}{k}$.

To fill in the picture, we also reported a number of computational
experiments that measure the runtime on graphs drawn from
$\model{n}{p}{k}$. These experiments point to a critical value for $p$
at which the \ea requires more time to find any $k$-cover, which
suggest an interesting direction for future theoretical work to
understand this phenomenon better.

\label{sec:conclusion}

\section*{Acknowledgements}
This work was supported by the National Science Foundation under grant
2144080 and by the Australian Research Council under grant
FT200100536.

\newcommand{\etalchar}[1]{$^{#1}$}

\end{document}